\documentclass[runningheads]{llncs}

\usepackage[T1]{fontenc}
\usepackage{graphicx}
\usepackage{amsmath,amssymb}

\begin{document}

\title{Hypercubes, Hyperplanes, and Constraint-Induced Complexity Collapse in Atomic Concept Learning}
\titlerunning{Hypercubes, Hyperplanes, and Complexity Collapse}

\providecommand{\orcidID}[1]{\textsuperscript{\,\texttt{[#1]}}}

\author{Irene Tsapara\orcidID{0009-0006-1819-2639}}
\authorrunning{I. Tsapara}

\institute{National University, San Diego, United States\\
\email{itsapara@nu.edu}}

\maketitle

\begin{abstract}
We revisit higher-arity atomic concept learning through the geometry of hypercubes and hyperplanes of ground instances. Our starting point is the observation that the ambient $r$-dimensional hypercube of ground atoms is not structurally uniform. Its logical complexity is organised by hyperplanes: every hyperplane other than the full diagonal collapses into finitely many elementary-equivalence classes, with a bound independent of the term depth, while the full diagonal is exceptional and its class count grows without bound. This asymmetry is not merely geometric; it reflects the reduction-theoretic structure of the concepts themselves.

Building on a higher-dimensional framework developed in the author's earlier work, we reinterpret these results through canonical simple concepts, minimal orderings, and representative reductions. This yields a taxonomy of hyperplane behaviour in higher dimensions and shows that complexity is localised rather than spread uniformly through the instance space. The paper includes a fully worked binary case, an explicit treatment of the ternary hypercube, and an unpacked account of the reduction machinery that drives the collapse. The three-dimensional case already exhibits the essential phenomenon: orthogonal families, partial diagonals, and the exceptional full diagonal. This geometric-logical perspective clarifies where complexity is concentrated in atomic concept learning and suggests a modern interpretation in terms of constrained hypothesis spaces and structured classification.

\keywords{higher-arity atomic concept learning \and hypercubes \and hyperplanes \and elementary equivalence \and minimal reductions \and logical constraints \and structured hypothesis spaces \and complexity localization \and neural classification}
\end{abstract}

\section{Introduction}

The learnability of logical concept classes depends not only on the size of the ambient instance space, but also on its internal structure \cite{Vassiliades-TsaparaThesis1997,TsaparaTuran1998}. Earlier work on exact learning and on atomic formulas with prescribed first-order properties showed that logical restrictions can sharply reduce learning complexity \cite{Angluin1988,TsaparaTuran1998}. In this paper we revisit that phenomenon geometrically, and we focus on the higher-arity setting developed in the author's doctoral thesis \cite{Vassiliades-TsaparaThesis1997}, henceforth cited simply as \emph{the thesis}.

Our starting point is the hypercube of ground instances generated by bounded-depth terms. For an $r$-ary predicate, this yields a discrete $r$-dimensional space whose points represent ground atoms \cite{Vassiliades-TsaparaThesis1997}. The key observation is that this space is structurally inhomogeneous. Complexity is not distributed uniformly across the hypercube. Most hyperplanes collapse into a bounded number of elementary-equivalence classes, while the full diagonal remains exceptional.

\paragraph{What ``finitely many'' means here.}
One point deserves emphasis at the outset, because without it the main theorem can be misread. For a fixed term depth $n$, the hypercube $\mathcal H_{r,n}$ is a finite set, so every family of concepts over it is trivially finite. The content of the classification is therefore a statement that is \emph{uniform in $n$}: on every hyperplane except the full diagonal, the number of elementary-equivalence classes is bounded by a constant that does not depend on $n$. On the full diagonal no such uniform bound exists, and the number of classes grows without bound as $n$ increases. Throughout the paper, ``finitely many'' should be read in this uniform sense. The geometric picture we develop is the reason such a bound exists off the diagonal and fails on it.

\paragraph{Roadmap.}
Section~\ref{sec:picture} builds the geometric picture informally, with no formal machinery, starting from the binary lattice of ground atoms. Section~\ref{sec:prelim} fixes definitions. Section~\ref{sec:binary} works the binary case out completely; a reader who follows only that section will already have the essential idea, since every concept there is a quadrant, a ray, or a point, and the diagonal is visibly the odd one out. Section~\ref{sec:reduction} explains the reduction machinery that converts this geometry into a statement about elementary equivalence. Section~\ref{sec:classification} states and proves the classification theorem, Section~\ref{sec:ternary} works the ternary case, and Section~\ref{sec:diagonal} isolates why the diagonal resists the argument. Sections~\ref{sec:monolith}--\ref{sec:learning} give the monolithic and learning-theoretic readings.

\paragraph{Why the asymmetry matters.}
The split between regular and exceptional regions matters for two reasons. First, it identifies where the logical complexity of atomic concept learning actually resides. Second, it suggests a learning-theoretic interpretation: structural constraints do not merely regularise globally, but collapse large regions of the hypothesis space while leaving a small set of highly interactive regions as the principal source of complexity.

The contribution of this paper is therefore a structural reading of higher-arity atomic concept learning. We recover the hypercube and hyperplane framework, make the diagonal exceptionalism explicit, and reinterpret the resulting collapse in a way that connects to contemporary discussions of inductive bias and constrained classification.

\section{The Geometric Picture, Informally}
\label{sec:picture}

Before any definitions, it is worth seeing what the objects look like. Everything in this paper takes place in a space built from one constant $a$ and one unary function symbol $f$. Iterating $f$ on $a$ produces the ground terms
\[
a,\quad f(a),\quad f(f(a)),\quad \dots,\quad f^{n}(a),
\]
so a ground term of depth at most $n$ is completely described by a single number: how many times $f$ has been applied. The term space is a line of $n+1$ points, indexed $0,1,\dots,n$.

Now take a binary predicate $P$. A ground atom $P(f^{i}(a),f^{j}(a))$ is determined by the pair $(i,j)$, so the set of all ground atoms is a square lattice. This is the picture in Figure~\ref{fig:2d_lattice}: the horizontal coordinate records the depth of the first argument, the vertical coordinate the depth of the second, and the highlighted line $i=j$ is the diagonal that will turn out to be exceptional.

\begin{figure}[t]
    \centering
    \includegraphics[width=0.62\textwidth]{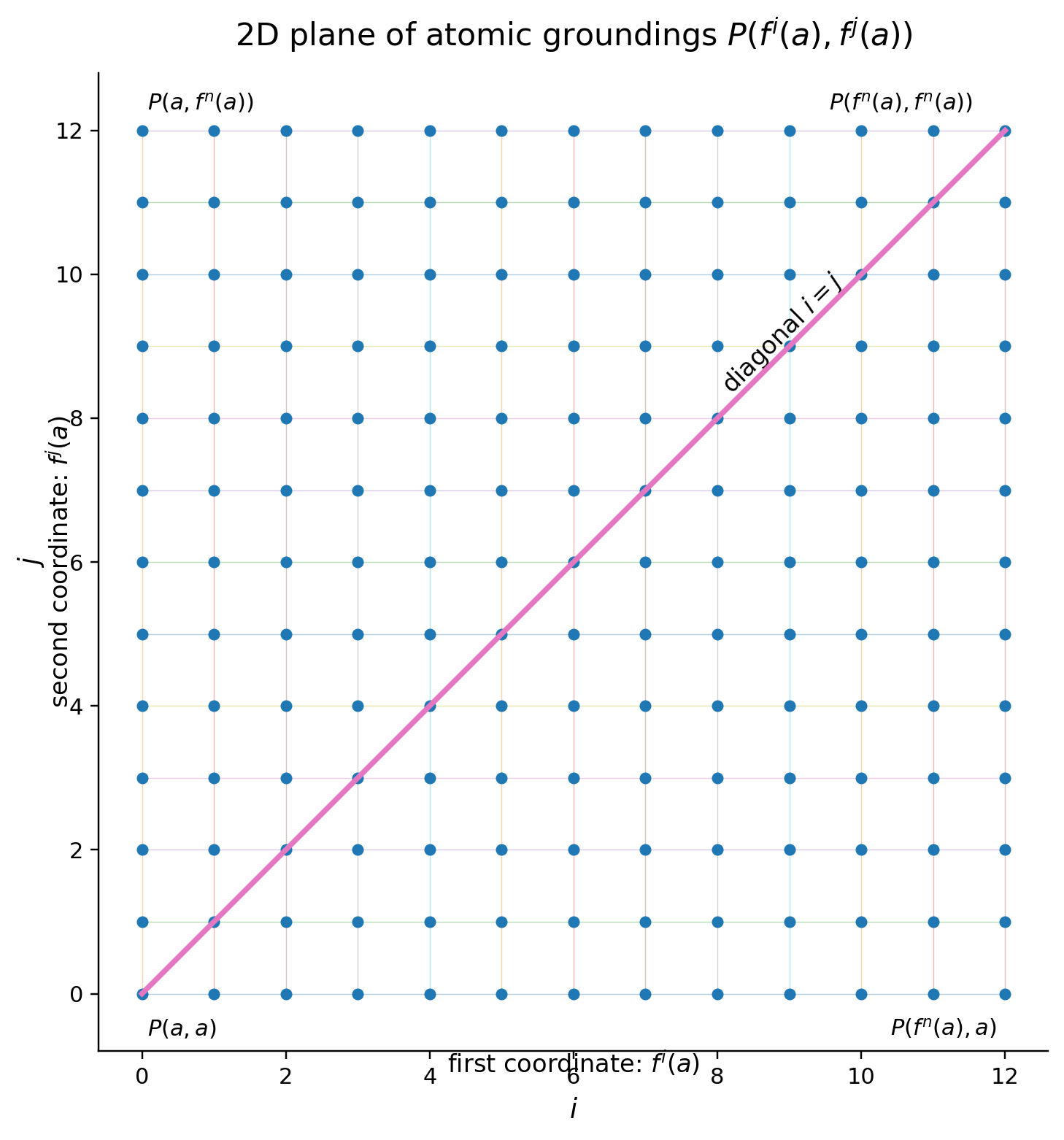}
    \caption{The binary case. Each lattice point is a ground atom $P(f^{i}(a),f^{j}(a))$; the coordinates are the depths of the two arguments. The highlighted line is the diagonal $i=j$.}
    \label{fig:2d_lattice}
\end{figure}

A lattice point is not a featureless location. Zooming in on one, as in Figure~\ref{fig:point_expansion}, shows that it is a tuple of structured terms: the point $(i,j)$ carries the pair $\bigl(f^{i}(a),\,f^{j}(a)\bigr)$, and each coordinate is itself an element of the term space. This is the observation that later becomes the \emph{monolithic} reading of the hypercube in Section~\ref{sec:monolith}. First-order constraints act across these structured coordinates, not merely on the lattice positions.

\begin{figure}[t]
    \centering
    \includegraphics[width=0.78\textwidth]{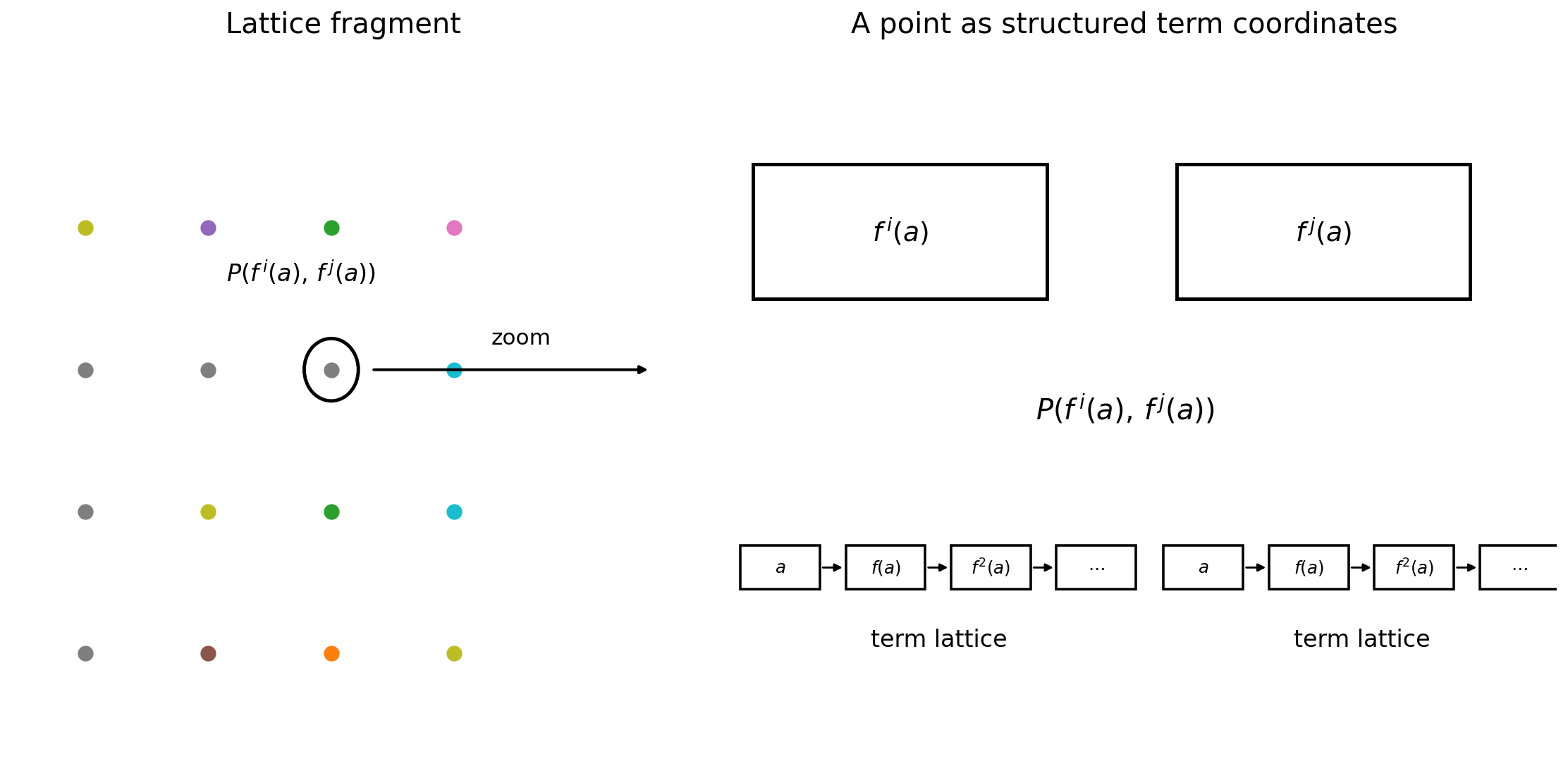}
    \caption{A point of the lattice, expanded. The hypervertex $(i,j)$ is the pair of structured terms $\bigl(f^{i}(a),f^{j}(a)\bigr)$, each drawn from the term space $T_n$; the arrows in the term lattices denote application of $f$.}
    \label{fig:point_expansion}
\end{figure}

Raising the arity raises the dimension. For a ternary predicate the ground atoms form a cube, shown in Figure~\ref{fig:3d_hypercube}, and for an $r$-ary predicate an $r$-dimensional grid. What changes in three dimensions is that coordinates can be identified \emph{partially}: one can require $i=j$ while leaving $k$ free, which is not possible when there are only two coordinates. The distinction between partial identification and total identification is invisible in the binary case and central in the general one.

\begin{figure}[t]
    \centering
    \includegraphics[width=0.55\textwidth]{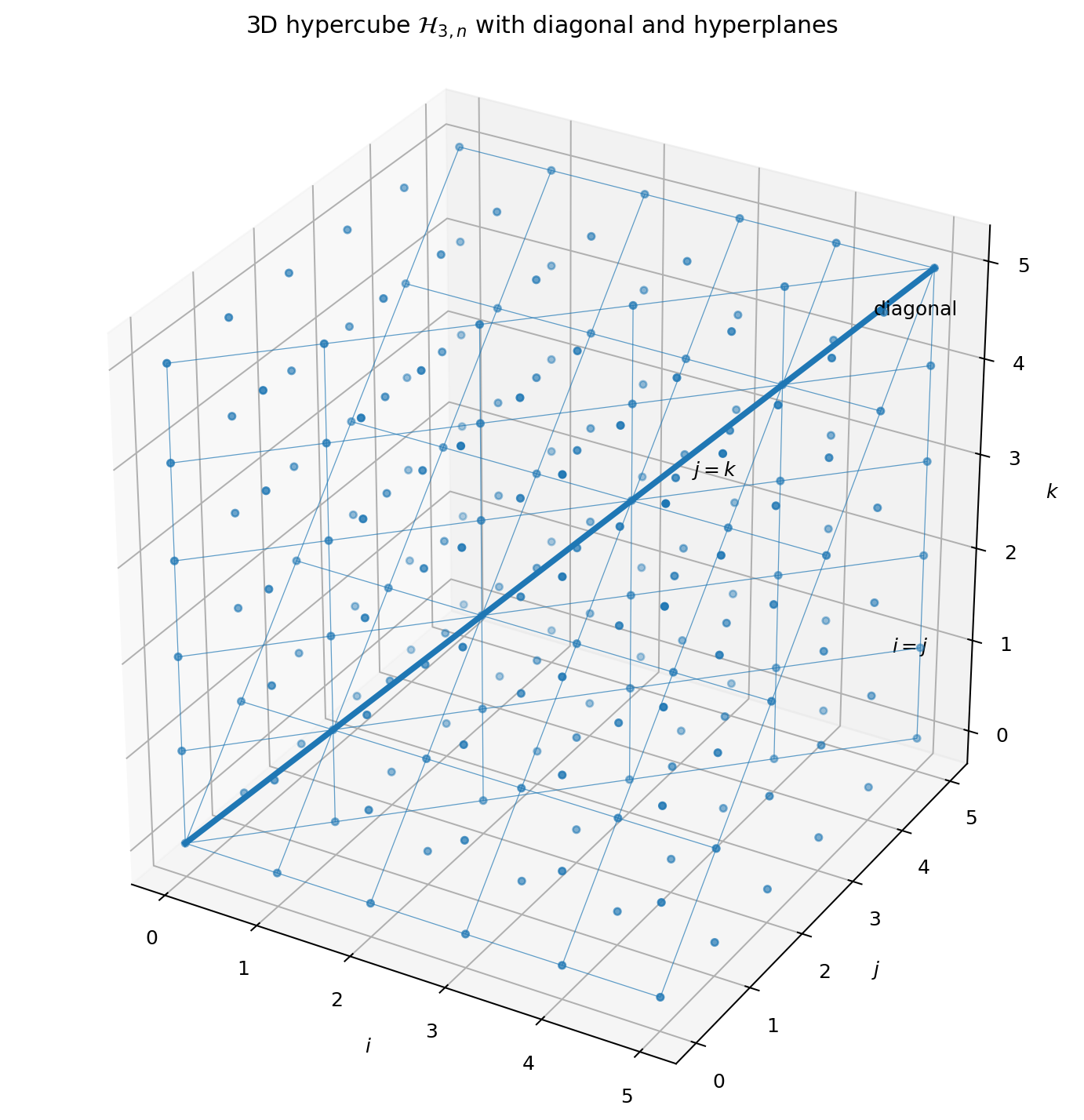}
    \caption{The ternary hypercube $\mathcal H_{3,n}$, with the partial diagonal $i=j$, the partial diagonal $j=k$, and the full diagonal $i=j=k$ indicated.}
    \label{fig:3d_hypercube}
\end{figure}

The claim this paper develops is that the grid is not uniform. Slicing it by coordinate conditions produces families of concepts whose logical complexity depends sharply on which slice was taken, and exactly one slice, the full diagonal, behaves differently from all the others.

\section{Preliminaries}
\label{sec:prelim}

\subsection{Terms, atoms, and the hypercube}

Let $f$ be a unary function symbol, $a$ a constant symbol, and let
\[
T_n=\{a,f(a),f^2(a),\dots,f^n(a)\}
\]
denote the set of ground terms of depth at most $n$. We write $f^{i}(a)$ for the element of depth $i$ and identify $T_n$ with $\{0,1,\dots,n\}$ whenever convenient.

For an $r$-ary predicate symbol $P$, define the ambient discrete space
\[
\mathcal H_{r,n}=T_n^{\,r},
\]
the \emph{hypercube of ground instances}. Each point $(t_1,\dots,t_r)\in\mathcal H_{r,n}$ represents the ground atom $P(t_1,\dots,t_r)$, and we use the coordinate notation $(i_1,\dots,i_r)$ for the corresponding depths.

A \emph{hypervertex} is a point of $\mathcal H_{r,n}$.

\begin{remark}[Terminology]
In the thesis, ``hypervertex'' (with ``hyperedge'') names a component of the hypergraph attached to a \emph{concept}, an object used in the reduction machinery of Section~\ref{sec:reduction}. Here the word denotes a point of the \emph{ambient} hypercube. The two usages describe different objects; where the hypergraph of a concept is meant, we say so explicitly.
\end{remark}

\subsection{Atomic formulas, subsumption, and concepts}

\begin{definition}[Atomic formulas and subsumption]
A \emph{term} over the present language is a variable, the constant $a$, or an application $f(s)$ of $f$ to a term $s$. An \emph{atomic formula} is an expression $P(s_1,\dots,s_r)$ whose arguments are terms; it is a \emph{ground atom} if no variable occurs in it. An atomic formula $A$ \emph{subsumes} an atomic formula $B$ if there is a substitution $\theta$, mapping variables to terms, with $A\theta=B$.
\end{definition}

For example, $P(f(x),y)$ subsumes the ground atom $P(f(a),f^{2}(a))$ via $\theta=\{x\mapsto a,\ y\mapsto f^{2}(a)\}$. It does \emph{not} subsume $P(a,a)$: substitution replaces variables by terms and can therefore only preserve or deepen the function structure of an argument, never remove the outer $f$. This monotonicity is worth keeping in mind when reading the figures: every specialisation arrow in this paper increases (or preserves) term depth, and no arrow can decrease it.

Given an atomic formula $A=P(s_1,\dots,s_r)$, let
\[
C_{A,n}=\{P(t_1,\dots,t_r)\in\mathcal H_{r,n}\;:\;A \text{ subsumes } P(t_1,\dots,t_r)\}
\]
be the \emph{concept represented by $A$} over $\mathcal H_{r,n}$; that is, the set of ground atoms obtainable from $A$ by substituting ground terms of $T_n$ for its variables. Let $\varphi$ be a first-order sentence over the language containing $P$. We study the class of concepts $C_{A,n}$ satisfying $\varphi$.

\subsection{Hyperplanes: orthogonal, diagonal, and the full diagonal}
\label{sec:hyperplanes}

A natural temptation is to use ``diagonal'' for total coordinate coincidence only, and ``non-diagonal'' for everything else. That conflicts with the thesis, where diagonal hyperplanes occur at every dimension. We therefore adopt the thesis vocabulary throughout and add one term for the exceptional case.

\begin{definition}[Hyperplanes]
A \emph{hyperplane} of $\mathcal H_{r,n}$ is a subset determined by a coordinate condition. It is
\begin{itemize}
\item \emph{orthogonal} if some coordinate is pinned to a fixed ground term, as in $i_1=p$;
\item \emph{diagonal} if some group of coordinates is required to be equal, as in $i_1=i_2$, with the remaining coordinates free.
\end{itemize}
A diagonal hyperplane is determined by a partition of the coordinate positions in which at least one block has size at least two. The \emph{full diagonal}
\[
\Delta_r=\{(i,i,\dots,i)\;:\;i\in T_n\}
\]
is the diagonal hyperplane of the one-block partition, in which all $r$ coordinates coincide.
\end{definition}

Diagonal hyperplanes thus exist at every dimension: for a ternary predicate there are $\binom{3}{2}=3$ diagonal hyperplanes of dimension two, namely $i=j$, $i=k$ and $j=k$, alongside the orthogonal ones. Only $\Delta_r$ identifies all coordinates at once. The classification below distinguishes $\Delta_r$ from every other hyperplane, orthogonal or diagonal.


\subsection{Elementary equivalence and minimal reductions}

\begin{definition}
Two concepts $C_{A,n}$ and $C_{B,n}$ are \emph{elementarily equivalent} if their associated relational structures satisfy the same first-order sentences in the relevant language fragment \cite{EbbinghausFlum1995}.
\end{definition}

\begin{definition}
A \emph{minimal reduction} of a concept structure is a reduced representative preserving its elementary theory. Two concepts with isomorphic minimal reductions are therefore elementarily equivalent.
\end{definition}

\begin{example}[The mechanism in miniature]\label{ex:ee}
A standard fact of finite model theory illustrates both notions in the simplest structure this paper uses, the term chain $a\to f(a)\to\cdots\to f^{m}(a)$ viewed as a finite successor structure. Sentences of quantifier rank $q$ cannot distinguish two such chains once both have length at least $2^{q}$ \cite{EbbinghausFlum1995}: an Ehrenfeucht--Fra\"iss\'e argument shows that any two sufficiently long chains satisfy exactly the same sentences of rank $q$, and are therefore elementarily equivalent relative to that fragment. A minimal reduction, in this miniature setting, replaces any chain of length $\ge 2^{q}$ by one fixed representative of length $2^{q}$: the elementary theory (relative to rank $q$) is preserved, and infinitely many chains collapse to one of boundedly many representatives, one per length below the threshold plus one for ``long''. The reductions used in the thesis operate on richer structures than bare chains, but the shape of the argument on hyperplanes other than the full diagonal is the same: beyond a bounded threshold determined by $\varphi$, additional depth along a free coordinate direction is invisible, so it is discarded by the reduction.
\end{example}

\section{The Binary Case, Worked in Full}
\label{sec:binary}

Everything essential is already visible when $r=2$, and in that case the concepts can be listed exhaustively. Let $A=P(s_1,s_2)$ be an atomic formula. Each $s_m$ is either a ground term $f^{p}(a)$ or a term $f^{p}(x)$ built on a variable, and there are at most two variables available. Substituting $x\mapsto f^{m}(a)$ turns $f^{p}(x)$ into $f^{p+m}(a)$, so a variable coordinate with $p$ applications of $f$ sweeps out the depths $p,p+1,p+2,\dots$ and never the depths below $p$. This single fact determines the shape of every binary concept.

\paragraph{The four shapes.}
Writing concepts as sets of coordinate pairs $(i,j)$ in $\{0,\dots,n\}^2$:
\begin{enumerate}
\item \textbf{Quadrants.} If $A=P(f^{p}(x),f^{q}(y))$ with $x$ and $y$ distinct, then
\[
C_{A,n}=\{(i,j)\;:\;i\ge p,\ j\ge q\},
\]
an axis-parallel region anchored at $(p,q)$. The two coordinates move independently.
\item \textbf{Diagonal rays.} If $A=P(f^{p}(x),f^{q}(x))$ with the \emph{same} variable in both positions, then
\[
C_{A,n}=\{(p+m,\,q+m)\;:\;m\ge 0\},
\]
a ray parallel to the main diagonal, offset by the fixed displacement $q-p$. The two coordinates move in lockstep.
\item \textbf{Axis rays.} If exactly one argument is ground, say $A=P(f^{p}(a),f^{q}(y))$, then $C_{A,n}=\{(p,j):j\ge q\}$, a ray inside a single column.
\item \textbf{Points.} If both arguments are ground, $C_{A,n}$ is the single hypervertex $(p,q)$.
\end{enumerate}

\paragraph{What the four shapes tell us.}
The taxonomy already separates the diagonal from everything else, and it does so for a reason that survives into higher arity. Shapes 1, 3 and 4 are described by \emph{independent} coordinate data: a lower bound on $i$, a lower bound on $j$, or a fixed value for one of them. Shape 2 is not. A diagonal ray is described by a \emph{relation between} the coordinates, the displacement $q-p$, and that relation is preserved by every application of $f$: advancing along the ray applies $f$ to both coordinates simultaneously and returns the displacement unchanged.

That is precisely the difference the classification theorem turns on. Off the diagonal, some coordinate direction remains free, and a minimal reduction can discard the depth information in that direction once it exceeds the finitely many thresholds that the constraint $\varphi$ can detect; the reduced structures then fall into a bounded family. On the diagonal, there is no free direction to discard, and the displacement is an invariant that distinguishes structures from one another without bound. As $n$ grows, the number of available displacements grows with it.

\begin{remark}[Where the uniformity enters]
This is where the clarification of Section~1 becomes concrete. For fixed $n$ there are only finitely many displacements, so the diagonal carries finitely many classes; the point is that this count is not bounded independently of $n$, whereas the counts attached to quadrants, axis rays and points are. The asymmetry is a statement about how the two families behave as $n\to\infty$.
\end{remark}

Figure~\ref{fig:hyperplane_classification} renders the contrast schematically, with the regular hyperplanes of Section~\ref{sec:hyperplanes} on the left and the full diagonal on the right.

\begin{figure}[t]
    \centering
    \includegraphics[width=0.72\textwidth]{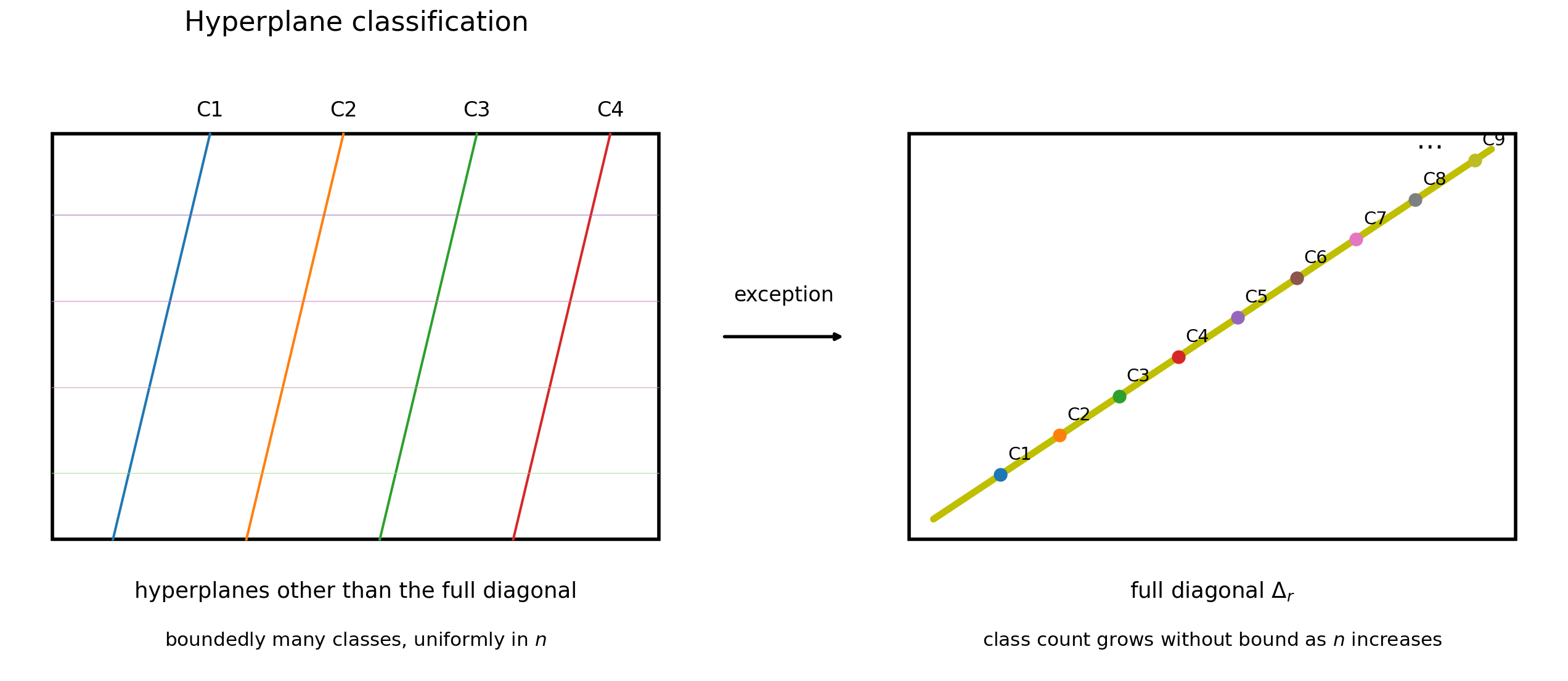}
    \caption{Hyperplane classification. Left: hyperplanes other than the full diagonal, whose induced concepts fall into a bounded number of structural classes. Right: the full diagonal, whose class count is not bounded uniformly in $n$.}
    \label{fig:hyperplane_classification}
\end{figure}

\section{Reduction Machinery}
\label{sec:reduction}

The classification is not a purely geometric fact; the geometry is what makes a reduction-theoretic argument go through. This section sets out that argument in outline, so that the theorem of Section~\ref{sec:classification} does not have to be taken on trust.

The machinery has three layers. The first is a notion of structure-preserving map between concept structures strong enough to preserve first-order theories. Strong homomorphisms preserve the relation and its complement; a \emph{reductive} homomorphism is a strong homomorphism onto a smaller structure that in addition respects the term structure of the coordinates. The basic fact is that reductive homomorphisms are elementary: a concept and its image satisfy the same sentences. Consequently, if two concepts admit a common reduction, they are elementarily equivalent.

The second layer makes the reduction canonical. To each concept one attaches a hypergraph recording which coordinate patterns are realised in it, and among all reductions of that hypergraph there is a minimal one, unique up to isomorphism. This minimal reductive hypergraph is the \emph{canonical concept}. The pivotal consequence is that the elementary-equivalence classes are exactly the classes of concepts sharing a canonical concept, which converts a question about first-order theories into a question about a finite combinatorial invariant.

The third layer is a counting argument on those invariants. Once elementary equivalence has been reduced to isomorphism of canonical concepts, bounding the number of classes on a region amounts to bounding the number of canonical concepts realisable there. This is where the geometry does its work: what a hyperplane fixes about the coordinates determines how much of the term structure survives reduction, and therefore how many canonical concepts are available.

\begin{remark}[Attribution]
The reduction apparatus reviewed here follows the thesis \cite{Vassiliades-TsaparaThesis1997}, which develops it after Elgueta \cite{Elgueta1997}; the reading of concepts as axis-parallel regions of a discrete space is likewise from the thesis. The presentation in this section is a summary, and the reader is referred there for the full statements.
\end{remark}

\section{Hyperplane Classification}
\label{sec:classification}

\begin{theorem}[Hyperplane Classification]
Let $\mathcal H_{r,n}=T_n^{\,r}$ be the hypercube of ground instances generated by one unary function symbol $f$, one constant symbol $a$, and an $r$-ary predicate symbol $P$, with terms of depth at most $n$. Let $\varphi$ be a first-order sentence over the language containing $P$.

Then for every hyperplane $H\neq\Delta_r$ of $\mathcal H_{r,n}$, the concepts induced on $H$ fall into at most $N(r)$ elementary-equivalence classes, where $N(r)$ depends only on $r$ and on $\varphi$ and not on $n$. On the full diagonal $\Delta_r$ no such uniform bound exists: the number of elementary-equivalence classes grows without bound as $n$ increases.
\end{theorem}

\paragraph{Shape of the argument.}
The proof has three steps. Step 1 partitions the hypercube into hyperplanes and records what each one fixes about the coordinates. Step 2 shows that when at least one coordinate direction survives unidentified, minimal reductions stabilise, giving a bound independent of $n$. Step 3 shows that on $\Delta_r$ the stabilisation fails, because the identification of all coordinates leaves an unbounded invariant behind.

\begin{proof}[Proof sketch]
\emph{Step 1: stratification.} The ambient hypercube is stratified by hyperplanes determined by coordinate equalities and by pinned coordinate values. Each hyperplane fixes some coordinate information and leaves the rest free. The binary case of Section~\ref{sec:binary} provides the base classification, in which the diagonal already appears as the unique exceptional region.

\emph{Step 2: stabilisation off the full diagonal.} Let $H\neq\Delta_r$. Then either some coordinate is pinned, or the identifying partition has at least two blocks, so at least one coordinate direction remains independent of the others. The role of minimal reduction is to remove redundant structure while preserving the elementary theory of the concept. Along an independent direction, depth information beyond the finitely many thresholds detectable by $\varphi$ is redundant and is discarded by the reduction, exactly as in Example~\ref{ex:ee}: once the free coordinate exceeds every depth threshold expressible at $\varphi$'s quantifier rank, a further application of $f$ yields a structure indistinguishable from one already produced, so it reproduces one of finitely many reduced configurations rather than a new one. The induced structures admit minimal reductions in which coordinate interactions remain separated, these reductions stabilise into a bounded family of canonical concepts, and by the correspondence of Section~\ref{sec:reduction} the induced concepts fall into a bounded number of elementary-equivalence classes. The bound depends on $r$ and $\varphi$ alone.

\emph{Step 3: failure on the full diagonal.} On $\Delta_r$ all coordinates coincide, so no independent direction remains. Applying $f$ advances every coordinate simultaneously and preserves the relative displacement pattern, which is therefore an invariant of the structure that the reduction cannot discard. Distinct displacement patterns yield non-isomorphic canonical concepts, and the number of available patterns grows with $n$. One obtains an unbounded sequence of pairwise non-equivalent reduced structures, so no bound uniform in $n$ can hold. \qed
\end{proof}

\begin{corollary}
Logical complexity in $\mathcal H_{r,n}$ is localised: every hyperplane other than $\Delta_r$ is structurally regular, and $\Delta_r$ concentrates the exceptional behaviour.
\end{corollary}

\paragraph{Scope.}
The theorem is stated in the thesis setting of one unary function symbol, one constant symbol, bounded term depth, and classification up to elementary equivalence. The purpose of the present paper is to isolate the geometric consequence of that analysis rather than to claim a more general result.

\paragraph{Beyond one unary function symbol.}
The restriction is not cosmetic, and it is worth stating exactly what depends on it. With a single unary $f$, the ground terms form a chain, so a ground atom is a tuple of natural numbers and the hypercube is a grid; the entire geometric reading rests on this. With $k$ unary function symbols the term space becomes a finitely branching tree, and the hypercube a product of trees: hyperplanes can still be defined by coordinate equalities, but a diagonal now identifies \emph{paths} rather than depths, and the displacement invariant of Section~\ref{sec:binary} is replaced by a richer word-valued invariant. With function symbols of arity two or more, terms are trees and the coordinate reading of a ground atom is lost altogether. Whether the localisation of complexity survives in either setting --- and in particular whether some analogue of the full diagonal remains the unique exceptional region --- is open, and we regard it as the natural next question raised by this work.

\section{The Ternary Case}
\label{sec:ternary}

The ternary hypercube is the first setting in which the distinction that drives the theorem becomes visible, because it is the first in which coordinates can be identified partially. In $\mathcal H_{2,n}$ there are only two coordinates, so identifying any two identifies all; in $\mathcal H_{3,n}$ the conditions $i=j$ and $i=j=k$ are genuinely different.

Consider the hyperplane $i=j$ in $\mathcal H_{3,n}$. It is a diagonal hyperplane of partial type: two coordinates are locked together, but $k$ is free. The induced concepts therefore retain an independent direction, Step 2 of the proof applies, and their reductions stabilise into a bounded family. Contrast the full diagonal $i=j=k$, where every coordinate is locked to every other and Step 2 has nothing to work with.

The geometry makes the relationship clear: the plane $i=j$ is itself a two-dimensional grid, with coordinates the shared value $i$ and the free value $k$, and inside it the line $i=j=k$ appears as its own main diagonal. The classification is thus recursive in a natural sense. A partial diagonal is a lower-dimensional copy of the same picture, regular except along its own diagonal.

\begin{remark}[A point the extended treatment must settle]
Because $\Delta_3\subseteq\{i=j\}$, the hyperplanes of $\mathcal H_{3,n}$ do not partition the cube: the exceptional line lies inside a hyperplane that the theorem classifies as regular. Two readings are available, and the paper should commit to one. Either hyperplanes are taken as a genuine stratification, with $\Delta_r$ removed from every hyperplane containing it, so that the regular regions are the differences $H\setminus\Delta_r$; or they are taken as an overlapping family, in which case the statement for $H=\{i=j\}$ requires an argument that the concepts induced on the plane are not distinguished by their restrictions to the line it contains. A treatment that omits the ternary case never has to face this question; the present paper does, and the stratification reading (removing $\Delta_r$ from each hyperplane containing it) is the one we adopt unless the thesis source dictates otherwise.
\end{remark}

\section{Why the Diagonal Is Exceptional}
\label{sec:diagonal}

The diagonal case differs because coordinate independence is lost. When all relevant coordinates coincide, substitutions and applications of the unary function feed back into the same coordinate pattern. This creates a recursive self-interaction that is not present elsewhere in the hypercube, and it is worth being precise about why the reduction machinery cannot absorb it.

Off the diagonal, the reduction discards depth information along a free direction. It can do so because that information is invisible to $\varphi$ beyond a bounded threshold: two structures agreeing up to the threshold agree on all sentences $\varphi$ can express, so their canonical concepts coincide. The bound on the number of classes is then a bound on the number of threshold configurations, which depends on $r$ and $\varphi$ but not on how deep the terms are allowed to go.

On the diagonal there is no free direction, and the surviving invariant is relational rather than positional. The binary case of Section~\ref{sec:binary} shows the phenomenon concretely: a diagonal ray generated by $P(f^{p}(x),f^{q}(x))$ carries the displacement $q-p$, and applying $f$ (that is, substituting $x\mapsto f(x)$) advances both coordinates together, returning the displacement unchanged --- this is what ``substitutions feed back into the same coordinate pattern'' means. Two rays of different displacement are not carried onto one another by any reduction, since the displacement is definable from the structure; and the number of realisable displacements grows with $n$. Reduction cannot discard the invariant without changing the elementary theory. The reduction process therefore need not terminate in a bounded family of elementary types, and the diagonal is the natural location of persistent complexity in the hypercube.

\section{A Monolithic Interpretation}
\label{sec:monolith}

We use the term \emph{monolithic} to emphasise that the hypercube $\mathcal H_{r,n}=T_n^{\,r}$ is not merely a flat Cartesian grid. Each coordinate belongs to the structured term space $T_n$, so each point is itself assembled from structured components, as Figure~\ref{fig:point_expansion} showed. The ambient space is a product of structured term spaces, and first-order constraints act across these layers simultaneously.

Figure~\ref{fig:monolith} shows the effect at the level of formulas rather than points. Starting from an atomic formula, substitution generates a hierarchy of specialisations, each a separate node, with arrows recording which substitution produced which. The hierarchy is the syntactic counterpart of the geometric nesting seen in Section~\ref{sec:ternary}: specialising a formula moves to a lower-dimensional region of the hypercube, and identifying two variables moves onto a diagonal.

\begin{figure}[t]
    \centering
    \includegraphics[width=0.92\textwidth]{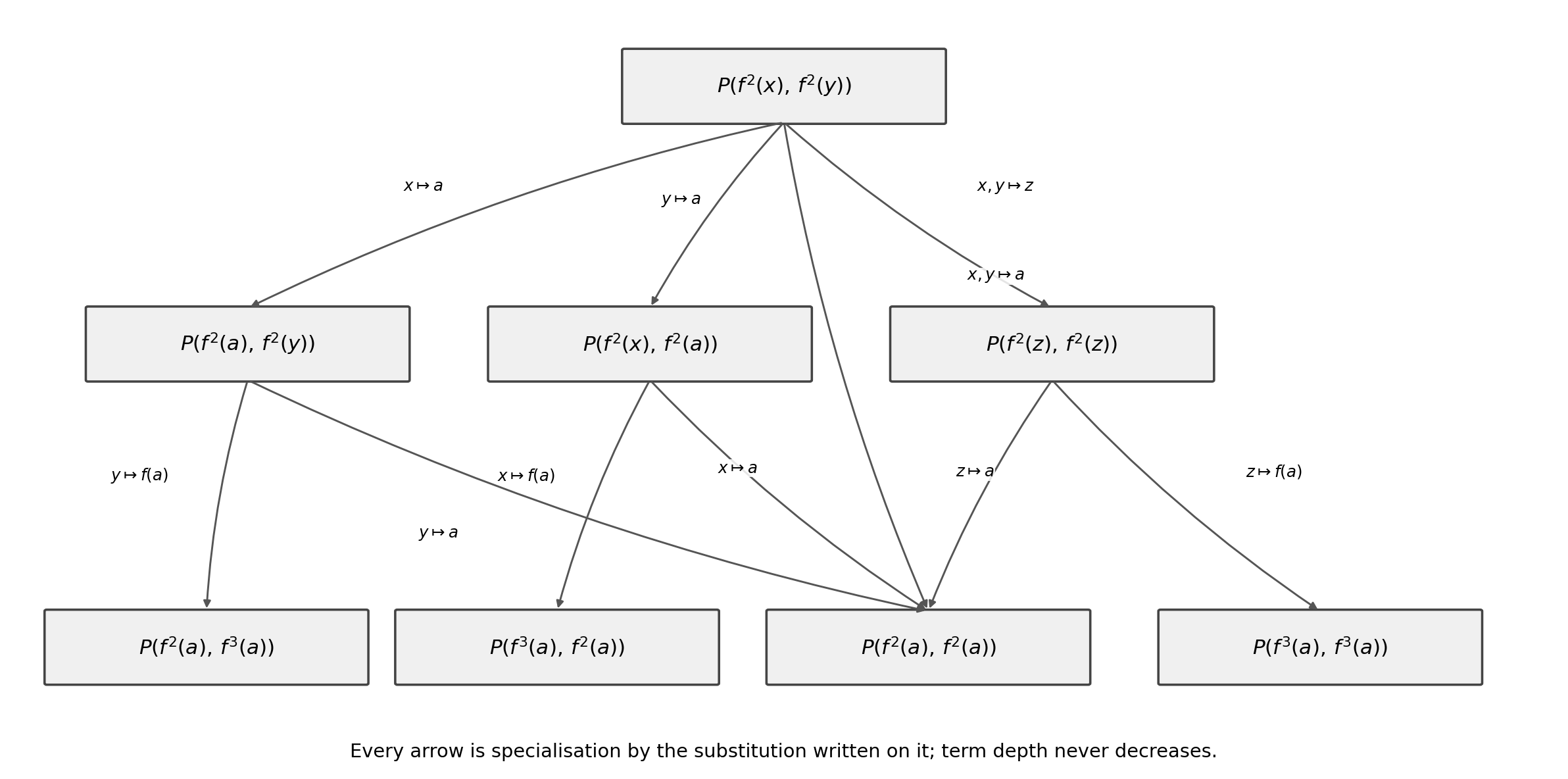}
    \caption{Monolith generated from $P(f^{2}(x),f^{2}(y))$. Every arrow is a specialisation by the substitution written on it, so along every arrow term depth is preserved or increased, as required by the definition of subsumption; identifying the two variables ($x,y\mapsto z$) moves onto a diagonal hyperplane, while instantiating a variable by a ground term moves onto an orthogonal one. The node $P(f^{2}(a),f^{2}(a))$ is reached along several routes, one from each of its subsuming formulas.}
    \label{fig:monolith}
\end{figure}

In this view the classification theorem says that first-order constraints collapse most of the ambient space into a bounded number of structural types, with the full diagonal remaining the unique source of persistent complexity \cite{Vassiliades-TsaparaThesis1997}.

\section{A Learning-Theoretic Reading}
\label{sec:learning}

\begin{proposition}[Learning-Theoretic Interpretation]
The hyperplane classification theorem suggests that, away from the full diagonal, the effective hypothesis space is controlled by a bounded number of structural classes, whereas the diagonal retains higher expressive complexity.
\end{proposition}

\begin{proof}[Interpretive proof sketch]
By the Hyperplane Classification Theorem, regions other than $\Delta_r$ admit only a bounded number of elementary-equivalence classes. A learner that distinguishes concepts only up to elementary equivalence therefore encounters a collapsed hypothesis space in those regions, and the collapse does not degrade as the term depth grows.

The diagonal region does not admit such a bound, since its class count grows with $n$. The effective classification complexity is therefore asymmetrically distributed: bounded on regular hyperplanes and unbounded on the diagonal. \qed
\end{proof}

\begin{remark}
A classifier that respects the constraints imposed by $\varphi$ should require fewer effective degrees of freedom on regular regions than on the diagonal, where feature interactions are maximally coupled. In particular, a representative-based classifier off the diagonal would depend on the bound $N(r)$ of the classification theorem rather than on the ambient size $(n+1)^{r}$ of the hypercube; making this dependence precise, and quantifying it in sample-complexity terms, is deferred to future work. The schematic in Figure~\ref{fig:neural_bridge} summarises the chain: a structured input space, a first-order constraint, collapse of the hypothesis space off the diagonal, and a correspondingly smaller classifier.
\end{remark}

\begin{figure}[ht]
    \centering
    \includegraphics[width=0.85\textwidth]{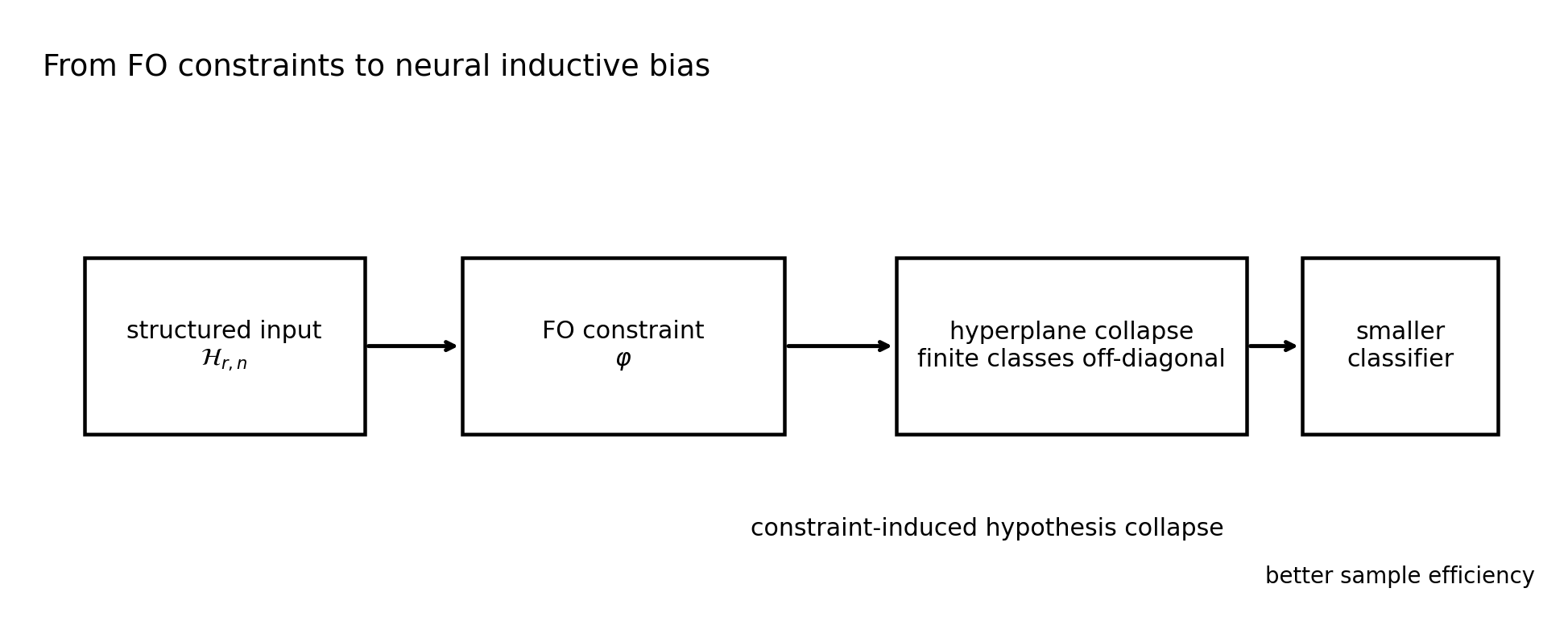}
    \caption{Constraint-induced collapse of the effective hypothesis space under first-order constraints.}
    \label{fig:neural_bridge}
\end{figure}

\section{Related Work}

Classical work on inductive generalisation established the logical background for structured concept learning. Plotkin's notes on inductive generalisation and his thesis introduced least general generalisation and the subsumption ordering on atomic formulas \cite{Plotkin1970,Plotkin1971,Plotkin1972}, and Angluin's query model supplied the learning-theoretic setting \cite{Angluin1988}. The thesis and the subsequent COLT paper belong to this line, showing that structural restrictions can make atomic formulas learnable \cite{Vassiliades-TsaparaThesis1997,TsaparaTuran1998}. Finite model theory supplies the notion of elementary equivalence and the classification techniques used here \cite{EbbinghausFlum1995}.

The present paper reinterprets that earlier theory geometrically, and the resulting picture has a counterpart in contemporary machine learning, where structural constraints are also used to shrink an effective hypothesis space rather than to regularise it uniformly. Equivariant architectures build symmetry directly into the model class \cite{CohenWelling2016}; permutation-invariant architectures do the same for set-structured inputs \cite{Zaheer2017}; higher-order graph networks calibrate expressive power against a combinatorial hierarchy \cite{MorrisEtAl2019}; and semantic loss functions inject logical constraints into training \cite{XuEtAl2018}. What the present analysis adds is a setting in which the collapse induced by a constraint can be located exactly, and in which the residual complexity is confined to an identifiable region.

\paragraph{Relation to earlier work.}
The higher-arity structural ingredients of this paper originate in the thesis, including the hypercube and hyperplane framework, minimal reductions, and the exceptional role of the diagonal. What is new here is the unified geometric presentation, the explicit complexity-localisation viewpoint, and the learning-theoretic interpretation of the resulting structural collapse \cite{Vassiliades-TsaparaThesis1997,TsaparaTuran1998}.

\section{Future Work}

Future work will connect the hyperplane classification developed here to neural classification and to explicit upper-bound learning algorithms. The main hypothesis is that first-order or symmetry constraints collapse the effective hypothesis space on regular regions, while diagonal regions remain the primary source of expressive complexity. This suggests architectures that do not treat all regions of the feature space uniformly: constrained or lower-capacity components may suffice off the diagonal, while diagonal regions may require richer representations for higher-order feature interactions.

A corresponding upper-bound algorithm would first decompose the instance space into hyperplanes, classify regular regions using representatives of their boundedly many elementary-equivalence classes, and treat the diagonal separately as the exceptional high-complexity component. Such a decomposition would make the connection between logical and neural classification more precise, by allowing a neural model to approximate the same division between regular and exceptional regions. The broader goal is to test whether bounded structural collapse off the diagonal yields measurable gains in sample efficiency, parameter reduction, and generalisation.

\section{Conclusion}

This paper revisits higher-arity atomic concept learning through the geometry of hypercubes and hyperplanes of ground instances. The main conclusion is that complexity is not uniformly distributed across the hypercube: every hyperplane other than the full diagonal collapses into boundedly many elementary-equivalence classes, with a bound independent of the term depth, while the full diagonal is exceptional and its class count grows without bound. The binary case makes the mechanism visible in four concept shapes, of which only the diagonal rays are described by a relation between coordinates rather than by independent coordinate data. The ternary case shows that the picture is recursive, with each partial diagonal a lower-dimensional copy of the same phenomenon. Off the diagonal, coordinate separation lets minimal reductions discard depth information and stabilise; on the diagonal, coordinates coincide, a displacement invariant survives every reduction, and complexity persists. The paper therefore recovers the higher-arity learnability framework in geometric form and gives it a logic-first interpretation as constraint-induced structural collapse.

\end{document}